\documentclass{llncs}
\usepackage{graphicx}
\usepackage{booktabs} % For formal tables
\usepackage{amsfonts}
\usepackage{pgfplots}
\usetikzlibrary{patterns}
\usepackage{fancyhdr}
\usepackage{comment}
\usepackage{capt-of}
\newtheorem{prop}{Proposition} 
\usepackage[ruled]{algorithm2e} % For algorithms

\SetAlFnt{\small}
\SetAlCapFnt{\small}
\SetAlCapNameFnt{\small}
\SetAlCapHSkip{0pt}
\IncMargin{-\parindent}
\begin{document}
\title{Disguised-Nets: Image Disguising for Privacy-preserving Outsourced Deep Learning\thanks{Partially supported by the National Science Foundation under Grant 1245847}}

\author{Sagar Sharma, Keke Chen\inst{}}
 \institute{Data Intensive Analysis and Computing (DIAC) Lab, Kno.e.sis Center, \\Wright State University
\email{\{sharma.74,keke.chen\}@wright.edu}}

\maketitle
\begin{abstract}
Deep learning model developers often use cloud GPU resources to experiment with large data and models that need expensive setups. However, this practice raises privacy concerns. Adversaries may be interested in: 1) personally identifiable information or objects encoded in the training images, and 2) the models trained with sensitive data to launch model-based attacks. Learning deep neural networks (DNN) from encrypted data is still impractical due to the large training data and the expensive learning process. A few recent studies have tried to provide efficient, practical solutions to protect data privacy in outsourced deep-learning. However, we find out that they are vulnerable under certain attacks. In this paper, we specifically identify two types of unique attacks on outsourced deep-learning: 1) the visual re-identification attack on the training data, and 2) the class membership attack on the learned models, which can break existing privacy-preserving solutions. We develop an image disguising approach to address these attacks and design a suite of methods to evaluate the levels of attack resilience for a privacy-preserving solution for outsourced deep learning. The experimental results show that our image-disguising mechanisms can provide a high level of protection against the two attacks while still generating high-quality DNN models for image classification.
\end{abstract}

\vspace{-0.35cm}
\section{Introduction}
Deep Neural Networks (DNN) generate robust modeling results across diverse domains such as image classification, natural language processing, speech recognition, and recommendation systems ~\cite{liu17dnn}. However, DNN training is resource intensive and time-consuming. Model developers often utilize cloud GPU resources to train large-scale models. A major concern is the privacy of the sensitive data and the trained models that may be stolen, traded, or possibly used to explore the private training data. 

One possible approach to addressing the privacy issue in outsourcing DNN is by training the models over encrypted data. However, due to the large training data and expensive training process, cryptographic approaches are too expensive to be practical as shown in a recent study on training small scale neural networks ~\cite{mohassel17}. As a result, cryptographic approaches are practically limited to testing DNNs as in CryptoNets \cite{xie14}, which is much more simpler and less expensive than training a DNN model.

Differential Privacy (DP) has been applied for deep learning in a different setting ~\cite{reza15,abadi16}, where sharing, not hiding, the training data and model is the goal. Abadi et al. ~\cite{abadi16} consider the problem of building DNN models that do not leak to the model consumers the private information specific to any individual training examples. Shokri et al. ~\cite{reza15} study a similar problem where data contributors are distributed. These methods cannot be adapted to the outsourced training scenario where privacy of both data and model is of concern. Furthermore, DP mechanisms result in a significant tradeoff between utility and privacy. For example, with a practical privacy setting, such as $\epsilon=2$, Abadi et al. ~\cite{abadi16} report over $15 - 20 \%$ accuracy reduction in classifying the MNIST dataset. 

A few recent studies try to address the confidentiality problem with the training data. Fan et al. ~\cite{fan18} apply differential privacy in hiding pixel-level details in the sensitive images. However, it does not protect sensitive contents that involve the entire images. In a different work, Li et al.~\cite{li17} propose hiding private data by submitting locally learned shallower neural networks to the cloud for further learning. However, the results show that the content of the intermediate representations is visually identifiable. These studies generally fail in protecting from the simplest attack --- manual visual content inspection by an attacker.

In summary, existing privacy-preserving techniques for deep learning are either too expensive, not designed for the outsourced setting, or vulnerable to the simple visual attacks.

\textbf{Scope and contributions.} 

In this work, we thoroughly study the problems with existing candidate approaches for privacy-preserving outsourced deep learning and propose two novel attacks that a viable solution should address. Our work is focused on image training data and classification tasks. (1) The first attack is the \emph{visual re-identification attack}, where an attacker can visually identify the major content of a protected image. Due to the recent studies ~\cite{krizhevsky17} that shows DNN models have outperformed human experts in image recognition tasks, the attack can be further carried out automatically with a trained ``DNN examiner" model that tries to recognize the contents from protected images. (2) The second attack is the \emph{class-membership attack}. In this attack, an adversary is able to use the model and access the model output, however, does not know the secret parameter settings of the data protection mechanism. The attack goal is to determine whether a class of images was included in the training data. Specifically, for classification tasks, this attack can be carried out by observing the distribution characteristics of the predicted labels for a set of images from the same class, e.g., different face images of the same person in face recognition. For a reasonably performing image classifier, the test images if similar to a class in the training data will have distinct output distributions, i.e., most labels will be of the same class, which is untrue if the test images are dissimilar from any of the training data classes. Most methods that expose the learned models or their outputs are vulnerable to this attack. 

To address these two attacks in the outsourced setting, we present a novel image disguising mechanism that protects both the training data and the learned models. The intuition is that \emph{with appropriately transformed images, the powerful deep learning techniques can still pick up the unique topological/geometric features preserved in the transformed spaces to effectively distinguish the classes of the transformed images}. Our image disguising mechanism combines block-wise permutation and multidimensional transformation to achieve excellent levels of \emph{visual privacy}, an empirical measure designed to evaluate the effectiveness of visual re-identification attack, and strong resilience to the model-based class-membership attack. We summarize our contributions as follows:

\begin{enumerate}
\item We identify two attacks that a viable privacy-preserving solution should consider in outsourced deep learning: the visual re-identification attack and the class-membership attack. We also propose empirical methods for evaluating the effectiveness of such attacks.

\item We design a suite of image disguising mechanisms for image-based DNN learning in the outsourced setting that thwarts both the visual re-identification and the class-membership attacks while preserving information in the transformed space for deriving high-quality models. 
\item We conduct extensive experimental evaluations on several public datasets to show the trade-offs of related parameter settings for the image disguising mechanisms and their resilience to the identified attacks. 
\end{enumerate} 

Next, we briefly outline this paper. Section ~\ref{sec:framework} describes DNN outsourcing, the potential privacy threats, and shortcomings of some existing techniques that target the privacy issue. Then, it introduces the visual re-identification and class-membership attacks and describes the security assumptions for our work. In Section ~\ref{sec:core}, we introduce a suite of image disguising mechanisms that enable privacy-preserving deep learning in the outsourced setting. Section ~\ref{sec:security_analysis} analyzes the security of the proposed mechanisms. Section ~\ref{sec:experim} presents the results for the experimental evaluations of our privacy mechanisms in terms of model quality, resiliency against the two attacks, and related trade-offs. We refer to the most relevant related works in Section ~\ref{sec:related_work} and finally conclude the paper in Section ~\ref{sec:conclusion}.

\vspace{-0.35cm}
\section{ Attacks on Outsourced Deep Learning}\label{sec:framework}
In this section, we describe the setting of outsourced deep learning and define two types of attacks that a viable privacy-preserving solution has to address.

\textbf{General Framework.} DNN learning is resource and time intensive for massive data and larger and intricate architectures such as ResNet ~\cite{kaiming15}. Resource-constrained data owners outsource their data to public cloud providers such as AWS's elastic GPUs to benefit from their high-performance GPU computation resources. Figure ~\ref{fig:framework_generic} shows the general framework for outsourced deep learning: the data owner offloads her training data to the cloud provider and deploys the cloud provider's GPU resources in training complex DNN models. After training, the data owner can either download the learned model for local use or just upload newer testing data to the cloud for prediction. 

\begin{figure} [h]
\centering
\includegraphics[width= 0.45\linewidth]{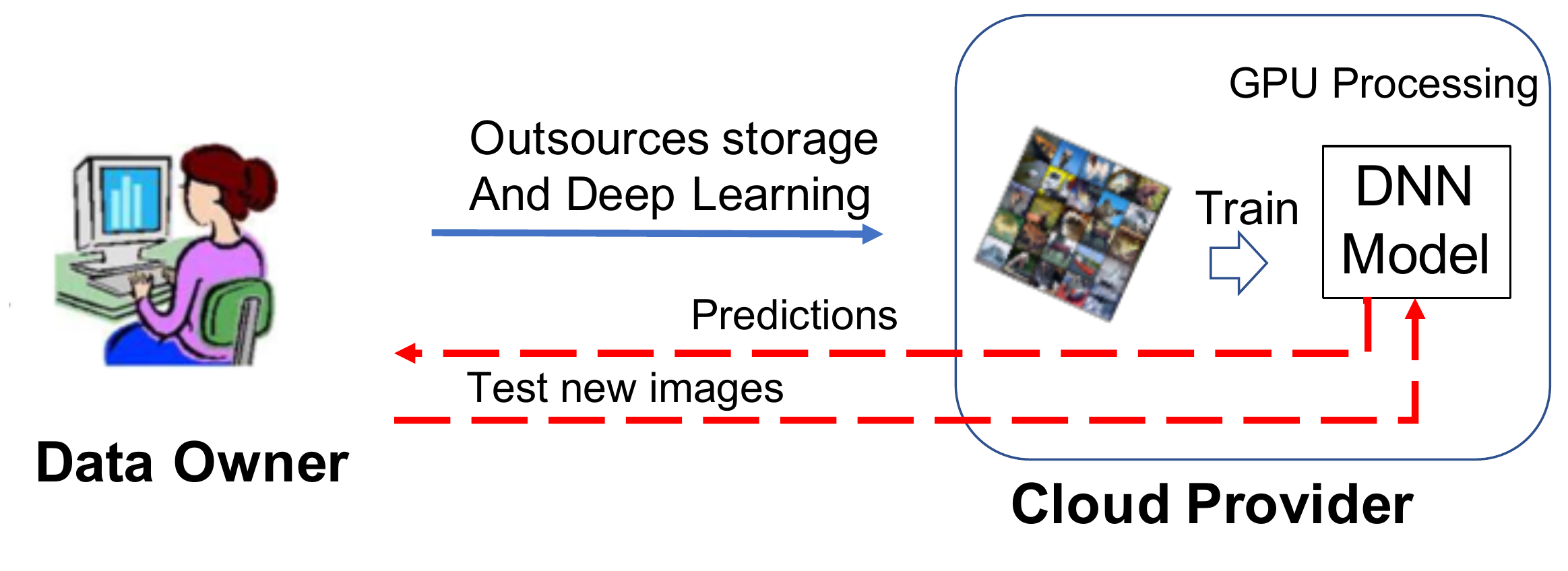}
\vspace{-0.3cm}
\caption{A data owner outsources her images to a cloud provider for storage and DNN modeling using GPU clusters.}
\label{fig:framework_generic}
\vspace{-0.55cm}
\end{figure}
%\vspace{-0.25cm}

Outsourcing DNN learning to the cloud provider introduces privacy risks when the training images may include sensitive images, as cloud providers are not fully trustable.  An adversarial party may investigate the uploaded images manually with human visual inspection to identify and analyze characteristics of sensitive objects. Furthermore, the adversary may exploit the learned models for its own use or launching model-based attacks as we will describe. Therefore, the data owner must either completely trust the cloud provider or deploy some data and model privacy protection mechanisms. 

\textbf{Problems with Existing Privacy-Preserving Approaches.} Due to the inherent cost of crypto approaches, they are not ideal candidates in hiding images from adversaries in outsourced deep learning as shown by ~\cite{mohassel17,xie14}. A few efficient data protection methods have been proposed recently to address the cost issue. Fan et al. ~\cite{fan18} rely on a differentially private (DP) mechanism to hide certain sensitive pixels (e.g., a person's face or a license plate number)  in the larger-context images. However, it depends on the predefined size of the sensitive objects in the images to determine the noise and $\epsilon$ privacy levels. This is difficult to apply in practice because of numerous possible variations of shapes, sizes, distances, and angles a sensitive object may be projected in different images. Often, many training data contains large size objects, such as human faces, that can be sensitive. Even for images containing many objects, as shown in Figure  ~\ref{fig:existing_app} (a), it is difficult to identify which sizes of an object are appropriate to protect. Furthermore, it is unclear whether blurring the targets is sufficient - the overall context may also reveal private information. 

In a different work, Li et al. ~\cite{li17} propose learning the first few layers of the target DNN model locally and outsourcing the intermediate representation to the cloud provider for further learning. Unfortunately, the intermediate representation when reconstructed at the cloud provider's sites is visually recognizable despite high Peak Signal-to-Noise Ratio (PSNR) (a metric used to quantify visual privacy in the paper) as seen in Figure ~\ref{fig:existing_app} (b). 

\begin{figure}
\vspace{-0.55cm}
  \centering
  \begin{tabular}{c @{\quad} c }
    \includegraphics[width=.30\linewidth]{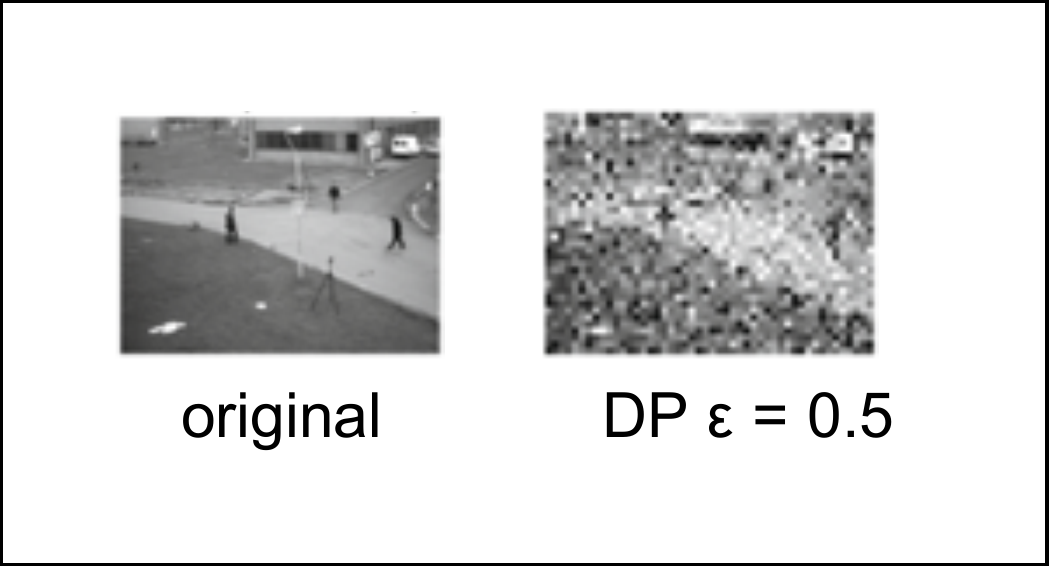} &
    \includegraphics[width=.30\linewidth]{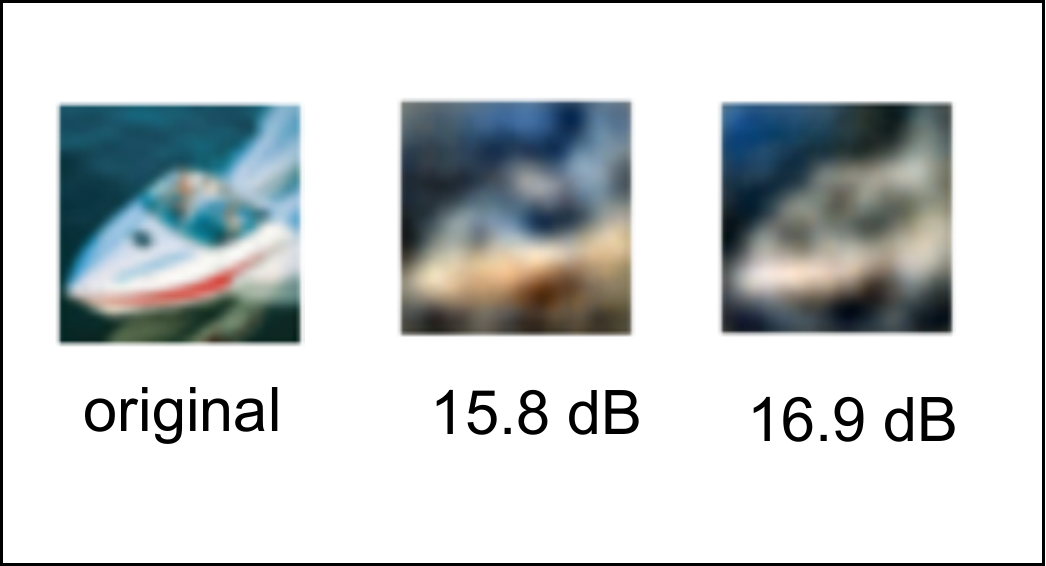} \\
    \small (a) & \small (b)
  \end{tabular}
  \vspace{-0.35cm}
  \caption{(a) DP-based pixellation of images ~\cite{fan18} reveals the global properties of the images making them distinguishable.  (b) The reconstructed images in PrivyNet ~\cite{li17} resemble the original images despite high PSNR values.}
  \label{fig:existing_app}
  \vspace{-0.45cm}
\end{figure}

In summary, there is no practical method for protecting the confidentiality of training data yet in outsourced deep learning. The cryptographic methods are still too expensive, and the two recent methods that avoid encryption \cite{li17,fan18} do not protect the images from visual inspection. 

\vspace{-0.3cm}
\subsection{Adversarial Model and Attacks}\label{sec:threat_model}
In the general framework we described earlier, we consider the adversaries can compromise the cloud infrastructure thus see the training data and the learned model if no protection mechanism is applied. In addition, we assume that adversaries can use the model as a black box - feeding the model with testing images and getting the prediction labels - even if the training data and the model is protected with some mechanism while the output labels are revealed. We will present two attacks specific to outsourced image-based deep learning: 1) Visual re-identification attack aimed at compromising the visual privacy of protected images, and 2) the class-membership attack aimed at exploiting a trained model in determining if a certain class of images was included in the training dataset. Then, we will elaborate on the security assumptions we will base our work on.

\vspace{-0.4cm}
\subsubsection{Visual Re-identification Attack.} 
We have shown that the existing work \cite{li17,fan18} are not effective at all in hiding sensitive images. An attacker can simply browse the blurred images to find out sensitive information.  Before we propose our own mechanism for image-based deep learning (Section \ref{sec:core}), we explore the basic requirement for protecting the privacy of image data: no adversary should be able to visually identify sensitive objects from the protected image data. We name this characteristic \emph{visual privacy}. We found that none of the existing metrics can precisely capture or define visual privacy. Use of pixel-level mean square error by Fan et al. ~\cite{fan18} and peak signal-noise ratio by Li et al. ~\cite{li17} do not capture the semantic understanding level that humans' visual perception can. 

We propose a DNN-based visual re-identification ``examiner" to serve as the agent of a human attacker. The recent advances in high-accuracy DNN models ~\cite{krizhevsky17} have shown that DNN models have exceeded human experts in image classification. Inspired by this, we propose to use a DNN models to impersonate the visual attackers to scan the protected images. We call such DNN models the ``DNN examiners''. Specifically, we can train a DNN examiner on the original training data and deploy it to distinguish the protected images. A high-accuracy result of the re-identification attack suggests the protection mechanism under scrutiny fails to maintain visual privacy. We define then \emph{visual privacy} as (1-Accuracy of DNN examiner in classifying the protected images).

\vspace{-0.4cm}
\subsubsection{Class-membership Attack.} 
Given a protection mechanism that thwarts the visual re-identification attack successfully, we need to eliminate any potential abuse of the exposed model trained on the protected images in exploring the training images. After carefully examining the outsourced deep learning scenarios, we identify a new class-membership attack that has not been defined or explored by the related work \cite{fan18,li17} yet. 

In the following, we design an attack that enables adversaries to learn whether a certain class of images was used as training examples by observing the target model's outputs. Figure \ref{fig:class_attack} shows the basic setting that adversaries can use (and thus explore) to try the model with any testing data in their hand and observe the model's outputs. The intuition is that a well-trained model should work nicely on records similar to the ones belonging to the training data classes whereas poorly on images from unrelated categories. For example, a face recognition DNN trained on 10 persons' face images must work much better on test images belonging to the same 10 individuals on test images belonging to others.  Let us denote the two categories of image classes as ``in-training" and ``out-training" classes respectively.

We can define this attack formally. Given a fully trained DNN model and known output labels $\{c_i| c_i \in C\}$, the adversary prepares a set of images, $\{t_i, i=1..m\}$, belonging to some class $c$ (a target class) that may or may not be one of the output labels. The adversary launches the class-membership attack to determine if $c\in C$, i.e. if the training dataset included images belonging to the target class $c$. The attacker's strategy is to characterize the output distribution $Pr(c'|\{t_i\})$ that will aid in inferring class memberships, where $c'$ represents the models prediction outputs. Test images belonging to an in-training class are consistently classified by the model to the same class with high probability for a reasonably good model; whereas test images belonging to an out-training class may see more uncertain outputs. Figure ~\ref{fig:class_attack} illustrates the idea of class-membership attack. A pointy histogram infers the target class (or a closely related class) was likely included in the training set whereas a flatter histogram suggests the target class was not likely included in the training set. Such distribution differences can be captured with entropy or Fano factor. Fano factor, similar to the variance-to-mean ratio (VMR), measures the index of dispersion and can be used in determining how two sets of observed occurrences are clustered or dispersed. Correspondingly, in-training examples of the same class will show smaller entropy or higher Fano factor than out-training examples. In experiments, we have shown that unprotected models are extremely vulnerable to class-membershup attack. 

\vspace{-0.5cm}
\begin{figure}[h]
\centering
\includegraphics[width= 0.65\linewidth]{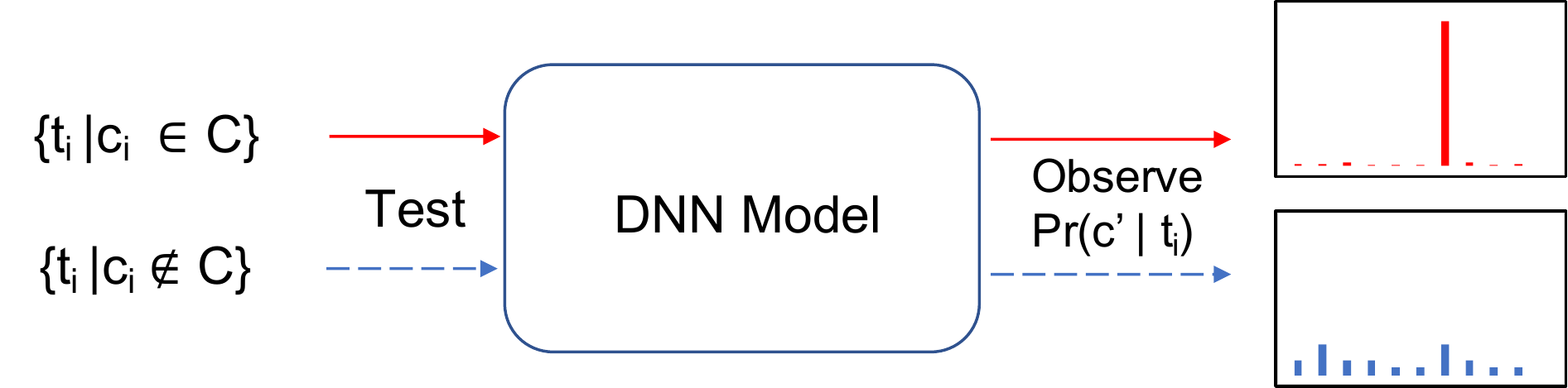}
\vspace{-0.15cm}
\caption{Class-membership Attack. Given a DNN model, an attacker analyzes the prediction outcomes of a series of testing images to determine if a certain class of images were included in the training set.}
\vspace{-0.4cm}
\label{fig:class_attack}
\end{figure}
\vspace{-0.3cm}

Expanding on the class-membership attack, an attacker may determine if a certain dataset was likely the training data, partially or fully. 

Any deep learning framework that exposes the model for sharing are subject to this attack. For example, the differentially private models trained with the techniques of Abadi et al. ~\cite{abadi16} and Shokri et al. ~\cite{reza15} are subject to the class-membership attack unless the model quality is low enough (due to the privacy setting), in which case the output distribution of in-training examples cannot be statistically distinguished from that of out-training examples. This, however, would fall short of the goal of developing useful models. 

\vspace{-0.4cm}
\subsubsection{Security Assumptions.}
Here, we will make some relevant security assumptions we will base our disguising mechanisms for image-based deep learning on: 1) We consider ciphertext-only attacks, i.e., any disguised image and its original image pair is unknown to the adversary; 2) The adversary possesses no prior knowledge of the images being outsourced, but they can use any images to explore the trained model; 3) All infrastructures and communication channels must be secure. We consider an honest-but-curious adversary, who may be interested in the contents and categorization (classes/labels) of the training images. The adversary might also want to misuse the models to distinguish or identify the domain of the training data with class-membership attacks.

\vspace{-0.35cm}
\section{Image Disguising for Deep Learning}\label{sec:core}
Our goal is to design a method to protect outsourced image-based deep learning from the two types of attacks described earlier. The current work will only address the challenges with whole-image based classification tasks. 

Figure ~\ref{fig:framework} depicts the Disguised-Nets framework. A data owner disguises her private images before outsourcing them to the cloud for storage and DNN learning. She transforms all of her images using one secure transformation key, $K$, which is comprised of the transformation types and the involved parameters. It is computationally difficult to guess the security key $K$ with brute-force and the transformed images do not leave sufficient information for adversaries to guess $K$ or launch the visual re-identification attack. She uses the cloud resources to train the DNN models from the transformed images with acceptable model quality.

\vspace{-0.15cm}
\begin{figure} [h]
%\vspace{-0.30cm}
\centering
\vspace{-0.35cm}
\includegraphics[width= 0.48\linewidth]{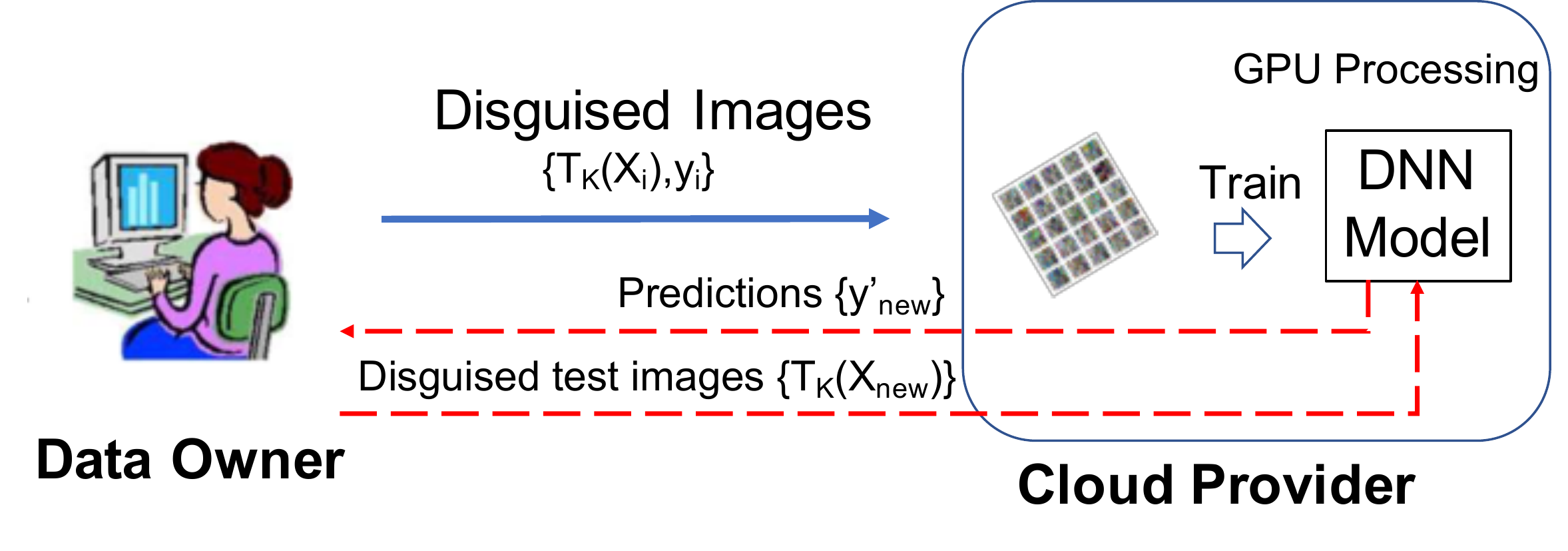}
%\vspace{-0.3cm}
\caption{Disguised-Nets: Image disguising framework for DNN learning.}
\label{fig:framework}
\vspace{-0.5cm}
\end{figure}
%\vspace{-0.10cm}

Specifically, assume the data owner owns a set of images for training, notated as pairs $\{(X_i, y_i)\}$, where $X_i$ is the image pixel matrix and $y_i$ the corresponding label. We formally define the disguising process as follows. Let the disguising mechanism be a transformation $T_K$, where $K$ is the secret key which depends on the selected perturbation techniques. By applying image disguising, the training data is transformed to $\{(T(X_i), y_i)\}$, which is used to train a DNN, denoted as a function $\mathbf{D}_T$, that takes disguised images $T(X)$ and outputs a predicted label $\hat{y}$. The models trained on images transformed with the image disguising mechanisms only work on transformed images and thus cannot be exploited with other image data as long as the transformation keys are secured. For any new data $X_{new}$, the model application is defined as $\mathbf{D}_T(T(X_{new}))$, the new data transformed with the same key $K$.

A remarkable characteristic of Disguised-Nets is that there is no need for one to alter or tailor the existing DNN architectures to make them compatible with the privacy mechanisms. One can simply import successful architectures such as ResNet and VGG to train the desired privacy-preserving models on the transformed data. In our opinion, this simplification is a great advantage over using traditional encryption or garbled circuit schemes which requires transforming the target DNN algorithms to their privacy-preserving versions, often a complex task, which results in expensive and impractical solutions.

The success of this approach depends on the transformation $T_K$ that preserves certain properties of the transformed data allowing DNN to learn the classification task. We consider a suite of image disguising mechanisms that can be combined with one another to achieve the desired level of privacy and utility. Candidate mechanisms must hide the \emph{visually identifiable} features of the images, i.e., attain good visual privacy and provide a sufficiently large key space to be resilient to ciphertext-only attacks. As a result, these mechanisms inevitably affect the quality of the learned DNNs. Therefore, finding the settings that provide both high security and model quality is crucial. While we have not theoretically justified the utility preserving mechanisms of these transformations yet, the empirical evaluation shows surprisingly good modeling results. 

\vspace{-0.35cm}
\subsection{Image Encoding and Partitioning} 
An image $X_{l \times m}$ with $lm$ pixels may have three RGB channels or just a single grayscale channel. We encode grayscale images as matrices of size ${l\times m}$ whereas the color images as three channel matrices of size ${3\times l \times m}$. The matrices might be partitioned into smaller \emph{blocks} for block-wise transformations to improve the visual privacy. In classification modeling, the image labels $c_i$ are mapped to $0, 1,  \dots$ without revealing their mapping to the actual classes. 

\begin{figure}
  \centering
  \vspace{-0.35cm}
  \begin{tabular}{c @{\quad} c @{\quad} c}
    \includegraphics[width=.32\linewidth]{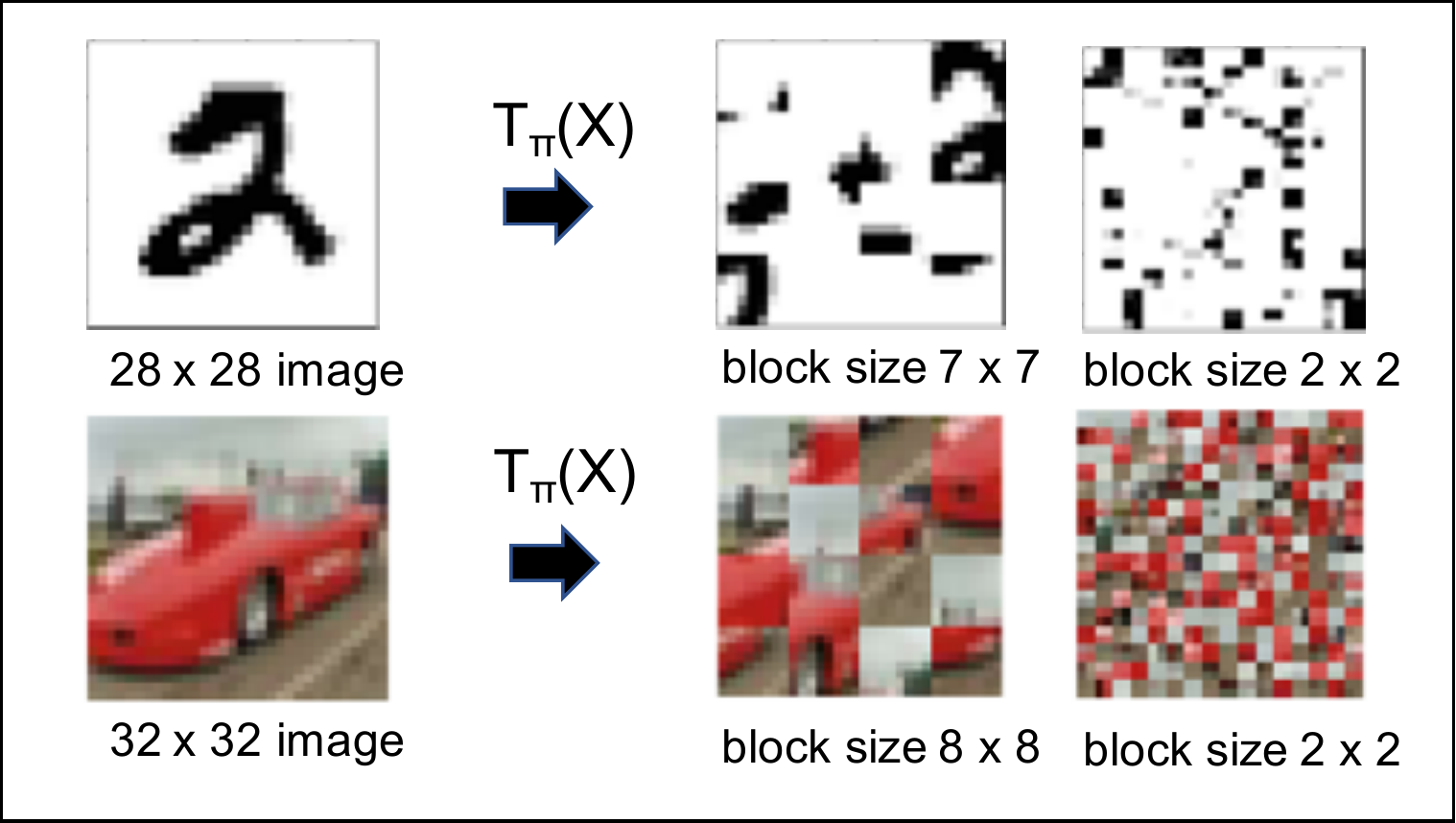} &
    \includegraphics[width=.32\linewidth]{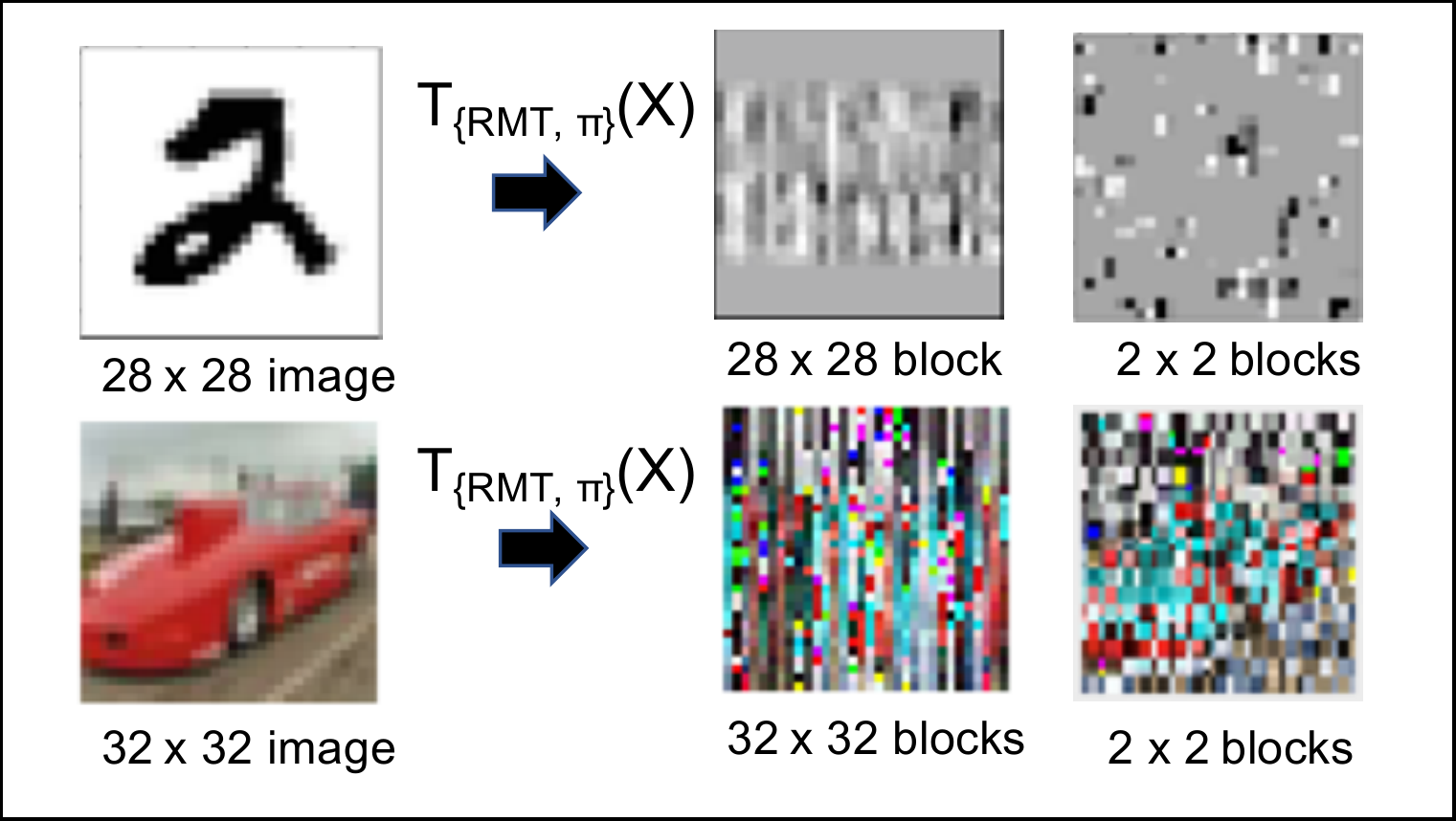} &
    \includegraphics[width=.32\linewidth]{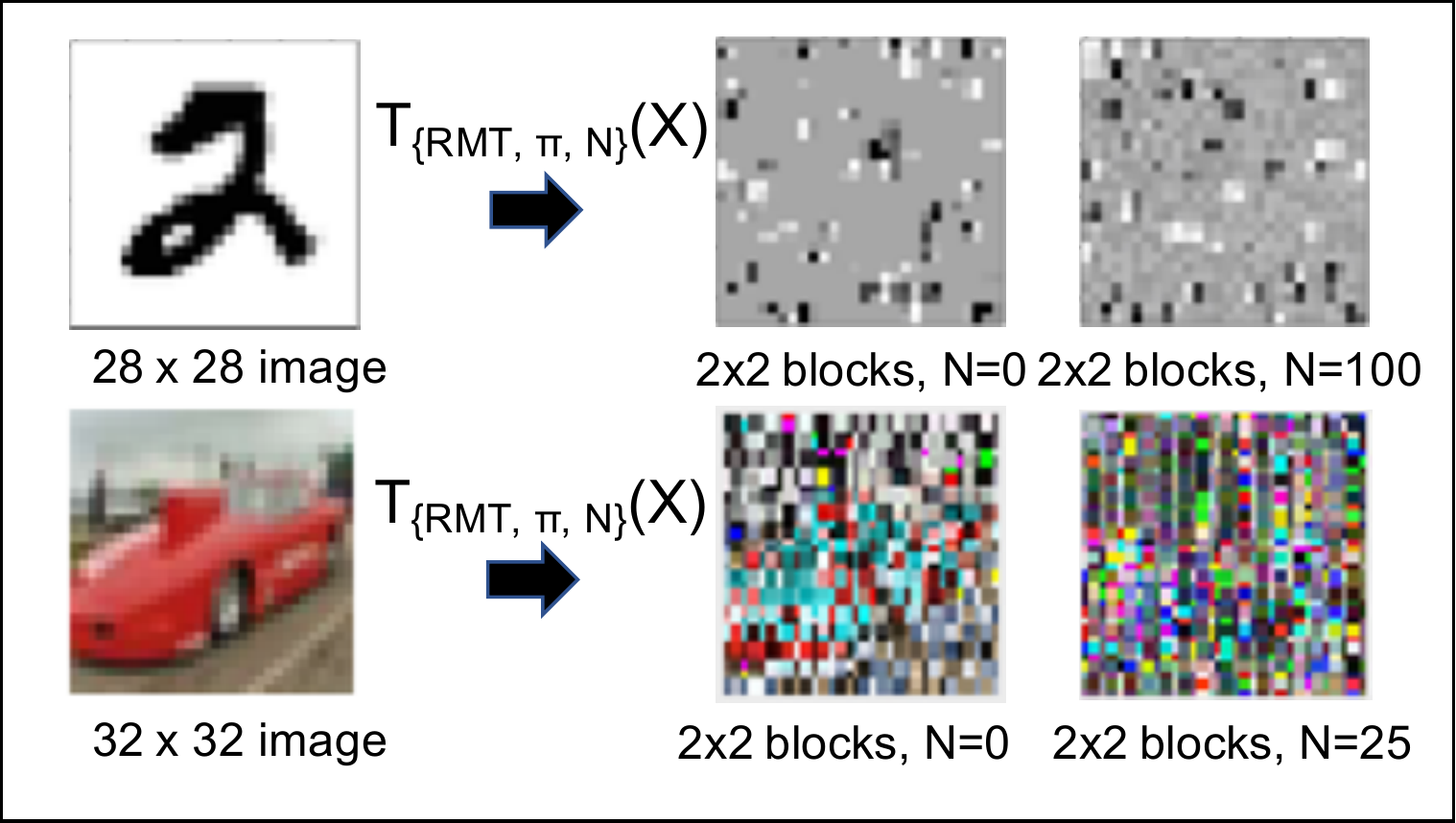} \\
    \small (a) Block-wise Permutation. & \small (b) Block-wise RMT. & \small (b) Block-wise RMT+Noise.
  \end{tabular}
  
  \caption{Different disguising mechanisms on MNIST and CIFAR-10 images. }
  \vspace{-0.55cm}
  \label{fig:disguises}
\end{figure}

\vspace{-0.45cm}
\subsection{Block-wise Permutation}\label{subsec:perm} 
The block-wise permutation simply partitions an image and re-arranges the image blocks randomly. An image $X_{l \times m}$ is partitioned into $t$ blocks of uniform size $ r\times s$. If we label the blocks sequentially as $v= <1,2,3,4, . . . t>$. A pseudorandom permutation of the image, $T_{\pi}(X)$, shuffles the blocks and reassemble the corresponding image accordingly. The permutation may break the global patterns of the images and achieve good visual privacy already. However, the block-wise characteristics such as boundaries, color, content shape, and texture of the original neighboring blocks may provide clues for adversaries to recover the original image - imagine the jigsaw puzzle! Figure ~\ref{fig:disguises} (a) shows an example. For large $t$, it might be difficult to apply such a jigsaw attack due to the vague similarity between block boundaries. In practice, however, the image size might be small, which leads to smaller settings of $t$ insufficient to protect from the jigsaw attack. Thus, we will need another layer of transformation to address the jigsaw attack. 
%\vspace{-0.4cm}

\vspace{-0.35cm}
\subsection{Randomized Multidimensional Transformations (RMT)}\label{subsubsec:tech_gdp}
\vspace{-0.1cm}
To address the block-wise jigsaw attack, we proceed with establishing a more resilient transformation mechanism that hides the visual attributes that aid in distinguishing the images from their transformed counterparts. For an image represented as a pixel matrix $X$, a general linear transformation can be defined as $G(X) = RX$, where $R_{ m \times m}$ is a random orthogonal matrix generated following the Haar distribution \cite{gallier00}, or a random projection matrix ~\cite{vempala05}. When an image is partitioned into $t$ blocks for random permutation, we will need a list of random matrices $\{R_i, i=1..t\}$, one for each image-block. The list of matrices $\{R_i\}$ acts as a secret key across the dataset and apply to the corresponding image-blocks. 

Such transformation is known to preserve (or approximately preserve by random projection) the Euclidean distance between columns of the matrix $X$. For non-image datasets, it is not possible to recover the original data from the perturbed data without certain prior knowledge of the data \cite{liu06tkde}, due to the large parameter space (we will discuss in Section \ref{sec:security_analysis}). However, for images, we found that without block-wise application, i.e., with one RMT for the entire image, RMT leaks information about sparse contents in images such as in MNIST dataset, as the zero-valued columns do not change after the transformation as shown in Figure \ref{fig:disguises} (b). Thus, it is often beneficial to combine block-wise permutation and RMT. Interestingly, when combining permutation and RMT, smaller block sizes preserve better model quality, while the block size may not matter much in protecting from visual re-identification attack, as shown by our experiments. 
 
To tackle the challenges presented by sparse images, we can also add noises into the transformation as $(X + \Delta)R$, where $\Delta$ is a random noise matrix, re-generated for each image (or image block) $X$, and drawn uniformly at random from $[0,N]$ where $N$ is the tunable noise level. Incorporation of additive noise before applying rotation perturbation converts the zero-pixel areas to noisy non-zero areas. Figure ~\ref{fig:disguises} (c) shows the effects of RMT on MNIST and CIFAR-10 datasets with and without the additional $\Delta$ noise. Note that the dense images, or even some sparse images, may have been well protected by block-wise permutation and RMT, and thus noise addition may not be needed as we will show in our experiments later.

\vspace{-0.45cm}
\section{Security Analysis}\label{sec:security_analysis}
In this section, we first present a proposition around the mathematical irreversibility of the proposed disguising mechanism under the assumption of ciphertext-only attacks. Next, we design the methods for the empirical assessment of the visual re-identification and class-membership attacks. 

\vspace{-0.3cm}
\subsection{Theoretical Analysis of Parameter Space Complexity}
With $X' = (X+\Delta)R$, we consider the ciphertext-only attacks scenario, i.e., the adversary has accesses to the ciphertexts, $X'$, without knowing any mapping pairs $X \rightarrow X'$. We present a proposition to show that enumerating the RMT matrices $R$ is computationally intractable. Assuming the $\Delta$ noise levels are relatively small and attackers can ignore it, a brute-force attack needs to enumerate all possible $R$ matrices to identify a valid one (e.g., by applying the visual re-identification attack). However, we show that the number of possible $R$ can be exponentially large for given parameters. 

\begin{prop} For values encoded in $h$-bit finite field, there are $O(2^{hm})$ candidate orthogonal matrices $R_{m\times m}$. 
\end{prop}
\begin{proof} With $h$-bit encoding, there are $p=2^h$ distinct values. The theory of orthogonal matrix group on finite fields states that there are $O(p^m)$ orthogonal matrices in $\mathbb{Z}_p^{m \times m}$ for a $p$-element field ~\cite{curtis84}. Hence, there are $O(2^{hm})$ orthogonal matrices.
\end{proof}
With the current setting we used in the experiments, e.g., $h = 32$ and $m=4$, enumerating the orthogonal matrices $R$ is computationally intractable. For random projection, as each element of the matrix is randomly drawn from a normal distribution, the number of possible matrices is even larger. 

We can extend the analysis to the case that block-wise RMT and permutation are combined. For simplicity, let each dimension of the matrix is partitioned into $r$ shares. Thus, there are $r^2$ matrix blocks, each of which has the size of $(m/r) \times (m/r)$. For one permutation of the blocks, there are $O(2^{hmr})$ combinations of matrices. Correspondingly, there are $O((r^2)!2^{hmr})$ combinations for all possible permutations. Therefore, block-wise partitioning further increases the complexity of the parameter space. 

In summary, the parameter space of the image disguising methods is large enough to address the brute-force ciphertext-only attack.

\vspace{-0.3cm}
\subsection{Empirical Assessments of Visual Re-identification and Class-membership Attacks}\label{subsec:empirical}
We design a set of tools for empirically assessing the effect of both the visual re-identification and class-membership attacks. For the visual re-identification attacks, we train a DNN examiner with the original training data and apply it to distinguish the disguised images from one another. We measure the overall accuracy of these applications. The result (1 - accuracy of DNN examiner) is used for empirical visual privacy.

For the class-membership attacks, attackers will apply the DNN models learned from the disguised data to original image data, possibly from other domains or the same domain. Note, without the secret keys, the adversary cannot transform the attack images. We mimic this attack by applying the disguised DNN to both in-training and out-training classes of images in their original form. We measure the class-wise Fano factors for both the in-training and out-training classes of images and see if the classes of the images are statistically distinguishable. Specifically, for a series of images $\{X_i, i=1..n\}$, the output label distribution over the classes forms a histogram. Let $n_{c_j}$ be the number of labels for $c_j$, $j=1..k$. The estimated mean of the distribution is $\mu = (\sum_{j=1}^k n_{c_j})/k$ and the estimated variance $\sigma^2$ is $(\sum_{j=1}^k(n_{c_j}-\mu)^2)/k$. The Fano factor is $\sigma^2/\mu$. We expect all in-training classes to have significantly higher Fano factor values than the out-training classes. In Section ~\ref{sec:experim}, we will empirically analyze the resilience of our image disguising mechanisms against both attacks. 

\vspace{-0.35cm}
\section{Experiments}\label{sec:experim}
\vspace{-0.15cm}

This experimental evaluation\footnote{Source code and scripts uploaded to https://github.com/datascale/DisguisedNets} has two specific goals. First, we will show that how effective the image disguising methods in preserving the model quality, with different parameter settings. Second, we show whether our methods are also resilient to the two attacks, with the empirical attack evaluation methods described in Section \ref{sec:security_analysis}.

\textbf{Datasets.} We test our disguising mechanisms with three prevalent DNN benchmarking datasets: MNIST and CIFAR-10 and a subset of face recognition dataset LFW faces. We use an additional dataset known as FASHION dataset in the class-membership attack evaluations on models trained with MNIST. MNIST (handwritten digits) and FASHION (fashion items) image-sets both consists of 60,000 training and 10,000 testing gray-scale $28 \times 28$ pixel-images with 10 classes. CIFAR-10 image-set consists of 50,000 training and 10,000 testing color-images of size $32 \times 32$ belonging to 10 classes. The subset of LFW faces dataset we use consists of a relatively smaller number (1,400 training and 150 testing) of color-images belonging to 12 classes. As LFW has images of a size larger than the images in CIFAR-10, we resize down the LFW images to 32x32 for assessing the class-membership attack.

\vspace{-0.4cm}
\subsection{Model Quality and Setup Cost}

Table ~\ref{tab:parameter_setting} details the mechanisms, block size, and additive noise level used for the datasets. We used a simple DNN architecture for MNIST implemented with TensorFlow, and the more powerful ResNet ~\cite{kaiming15} architecture implemented on PyTorch for CIFAR-10 and LFW datasets. For MNIST, we set the learning rate to 0.001 and train the network for 1,000 iterations. For CIFAR-10 and LFW, we adaptably adjust the learning rate from 0.1 to 0.001 as the models are trained for 350 iterations. We use an 8-GPU cluster to train the models and each experiment was carried out 5 times to capture the variances of results. Table ~\ref{tab:results} shows that the models trained on disguised images perform closely to the optimum models trained on the undisguised images.

%\vspace{-0.35cm}
\begin{table}[h]
	%\vspace{-0.4cm}
  	\centering
  	\scriptsize
  	\caption{Parameter settings and CNN Architectures.} \label{tab:parameter_setting}
	  \vspace{-.19cm}
  		\begin{tabular}{|c|c|c|c|c|}
		%\vpsace{-0.3cm}
  			\hline
  			Datasets &$Mechanisms$ & Block size & Noise Level& Architecture\\
  			\hline
			MNIST & block-wise RMT + Permutation &\{$7 \times 7\}$ &100 & Simple\\
			CIFAR-10 & block-wise RMT & \{$2 \times 2\}$ & 25 &ResNet\\
			LFW & block-wise RMT& \{$2 \times 2\}$&50&ResNet\\ 
			\hline
  		\end{tabular}
		%\vspace{-0.1cm}
 \end{table} 
%\vspace{-0.2cm}

%\vspace{-0.1cm}
\begin{table}[h]
	%\vspace{-0.3cm}
  	\centering
  	\scriptsize
  	\caption{Results of applying image disguising mechanisms.} \label{tab:results}
	  \vspace{-.19cm}
  		\begin{tabular}{|c|c|c|}
		\hline
		& \multicolumn{2}{c|}{Model Accuracy}\\
		Datasets &With Disguise & Without Disguise\\
  			\hline
			MNIST &96.6 +/- 0.4\% & 96.7 +/-0.2\% \\
			CIFAR-10 & 89.3\%+/-0.1\% & 93.4 +/-0.2\% \\
			LFW &90.6 +/- 1.3\%&94.3 +/-2.0\\
			\hline
  		\end{tabular}
		\vspace{-0.3cm}
  \end{table} 
%\vspace{-0.2cm}

The cost of the image disguising transformations are generally very cheap, (per image cost is less than 10ms) and can be comfortably done by any PC.
%The per-record disguising cost for the MNIST dataset for the optimal setting was less than $1$ ms whereas for the CIFAR-10 and LFW datasets was  $13$ ms on a 2.2 GHz I7 machine with 16 GB memory. The transformations resulted in image sizes of $8$ KB for MNIST whereas $33$ KB for CIFAR-10 and LFW dataset; roughly 2-5 times the original image sizes, negligible if compared to encrypting images with homomorphic encryption schemes \cite{paillier99,BGV12}. 

For the experiments in following subsections, we keep all the parameters in Table ~\ref{tab:parameter_setting} constant and vary the parameter under discussion unless noted otherwise. 

\vspace{-0.3cm}
\subsection{Effect of Parameter Settings on Model Quality}
\vspace{-0.15cm}
Our objective here is to understand the effect of different parameter settings on model quality. From Figure ~\ref{fig:proj_vs_orth} (left), it is clear that the DNN models were significantly more effective when applying RMT with the orthogonal matrices as compared to applying RMT with projection matrices for MNIST.  However, we observe the variation results in comparable model quality for CIFAR-10 and LFW. On the other hand, permutation of RMT blocks, which intuitively reduces the model quality, deteriorates the model quality negligibly for MNIST, moderately for LFW, and a bit alarmingly for CIFAR-10 as seen in Figure ~\ref{fig:proj_vs_orth} (right). We prefer orthogonal matrices for the optimum setting.

Figure ~\ref{fig:model_block_sizes} (left) shows that the model quality for LFW and CIFAR-10 datasets increases with smaller block sizes i.e. with the increasing number of blocks. However,  we do not observe much effect on MNIST. Intuitively, larger noise levels should degrade model quality. Figure ~\ref{fig:model_block_sizes} (right) shows the expected effect is prominent for the CIFAR-10 and LFW with significant degradation of model quality with increasing noise levels. Again, the effect is absent on the MNIST dataset, the model quality remaining steady with an increase in noise level.

\vspace{-0.1cm}
\begin{figure}[h]
\vspace{-0.3cm}
\centering
	\begin{tikzpicture}[scale=0.48]
	\begin{axis}[
 	ybar=-2pt, bar shift = 7pt,enlarge x limits=0.2,
 	bar width=12pt,
	log basis y={10},
 	ymin=50, ymax=100,
 	axis x line*=bottom, ylabel near ticks, yticklabel pos=left, yticklabel style={font=\Large},
	 y axis line style={opacity=50},
 	legend columns=2,legend style={at={(0.5,0.99)},draw=none,anchor=south,
 		font=\Large},
 	ylabel={Avg. Model Quality},
 	ylabel style={font=\Large},
 	symbolic x coords={MNIST Orth, MNIST Proj, blank, CIFAR Orth, CIFAR Proj,blank2, LFW Orth, LFW Proj},
	xticklabels={, MNIST, ,CIFAR-10,LFW},
	yticklabel=\pgfmathprintnumber\tick\%,yticklabel={\pgfmathparse{\tick}\pgfmathprintnumber{\pgfmathresult}\%}]
 	  \addplot[error bars/.cd, y dir=both,y explicit][draw=black, fill=blue,fill opacity=0.85,postaction={pattern=north east lines}] coordinates {
	  	(MNIST Orth,97.97)+-(0.3,0.3)};
	  \addplot[error bars/.cd, y dir=both,y explicit][draw=black, fill=yellow,fill opacity=0.85,postaction={pattern=grid}] coordinates {
		(MNIST Proj,59.39)+-(25.0,25.0)};
	 \addplot[error bars/.cd, y dir=both,y explicit][draw=black, fill=yellow,fill opacity=0.85,postaction={pattern=grid}] coordinates {
		(blank,0)};
	\addplot[error bars/.cd, y dir=both,y explicit][draw=black, fill=blue,fill opacity=0.85,postaction={pattern=north east lines}] coordinates {
	  	(CIFAR Orth,89.12)+-(0.02,0.02)};
	  \addplot[error bars/.cd, y dir=both,y explicit][draw=black, fill=yellow,fill opacity=0.85,postaction={pattern=grid}] coordinates {
		(CIFAR Proj,87.73)+-(0.56,0.56)};
	 \addplot[error bars/.cd, y dir=both,y explicit][draw=black, fill=yellow,fill opacity=0.85,postaction={pattern=grid}] coordinates {
		(blank2,0)};
		\addplot[error bars/.cd, y dir=both,y explicit][draw=black, fill=blue,fill opacity=0.85,postaction={pattern=north east lines}] coordinates {
	  	(LFW Orth,89.80)+-(1.30,1.30)};
	  \addplot[error bars/.cd, y dir=both,y explicit][draw=black, fill=yellow,fill opacity=0.85,postaction={pattern=grid}] coordinates {
		(LFW Proj,91.80)+-(0.01,0.01)};
	\legend{Orthogonal, Projection} 
 	\end{axis}
 	\end{tikzpicture}
	\begin{tikzpicture}[scale=0.48]
	\begin{axis}[
 	ybar=-7pt, bar shift = 6pt,enlarge x limits=0.2,
 	bar width=12pt,
 	ymin=50, ymax=100,
 	axis x line*=bottom, ylabel near ticks, yticklabel pos=left, yticklabel style={font=\Large},
	 y axis line style={opacity=50},
 	legend columns=2,legend style={at={(0.5,0.99)},draw=none,anchor=south,
 		font=\Large},
 	ylabel={Avg. Model Quality},
 	ylabel style={font=\Large},
 	symbolic x coords={MNIST Perm, MNIST NoPerm, blank, CIFAR Perm, CIFAR NoPerm, blank2,LFW Perm, LFW NoPerm},
	xticklabels={,MNIST, ,CIFAR-10, LFW},
	yticklabel=\pgfmathprintnumber\tick\%,yticklabel={\pgfmathparse{\tick}\pgfmathprintnumber{\pgfmathresult}\%}]
 	  \addplot[error bars/.cd, y dir=both,y explicit][draw=black, fill=blue,fill opacity=0.85,postaction={pattern=north east lines}] coordinates {
	  	(MNIST Perm,97.97)+-(0.35,0.35)};
	  \addplot[error bars/.cd, y dir=both,y explicit][draw=black, fill=yellow,fill opacity=0.85,postaction={pattern=grid}] coordinates {
		(MNIST NoPerm,97.08)+-(0.92,0.92)};
	 \addplot[error bars/.cd, y dir=both,y explicit][draw=black, fill=yellow,fill opacity=0.85,postaction={pattern=grid}] coordinates {
		(blank,0)};
	\addplot[error bars/.cd, y dir=both,y explicit][draw=black, fill=blue,fill opacity=0.85,postaction={pattern=north east lines}] coordinates {
	  	(CIFAR Perm,60.07)+-(0.7,0.7)};
	  \addplot[error bars/.cd, y dir=both,y explicit][draw=black, fill=yellow,fill opacity=0.85,postaction={pattern=grid}] coordinates {
		(CIFAR NoPerm,89.12)+-(0.02,0.02)};
			 \addplot[error bars/.cd, y dir=both,y explicit][draw=black, fill=yellow,fill opacity=0.85,postaction={pattern=grid}] coordinates {
		(blank2,0)};
		\addplot[error bars/.cd, y dir=both,y explicit][draw=black, fill=blue,fill opacity=0.85,postaction={pattern=north east lines}] coordinates {
	  	(LFW Perm,73.0)+-(0.71,0.71)};
	  \addplot[error bars/.cd, y dir=both,y explicit][draw=black, fill=yellow,fill opacity=0.85,postaction={pattern=grid}] coordinates {
		(LFW NoPerm,90.60)+-(1.34,1.34)};
	\legend{Permutation,No permutation} 
 	\end{axis}
 	\end{tikzpicture}
	\vspace{-0.2cm}
 	\caption{Effect on Model Quality: Orthogonal vs. Projection (left). Permutation (right)}
	\vspace{-0.3cm}	
 	\label{fig:proj_vs_orth}
\end{figure}
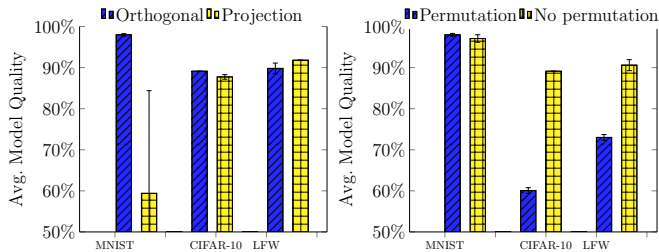
\vspace{-0.4cm}
\begin{figure}[h]
\vspace{-0.1cm}
\centering
\begin{tikzpicture}[scale=0.48]
 	\pgfplotsset{every axis legend/.append style={font=\small},every node near coord/.append style={font=\Large}}
 	\begin{axis}
 	[ymin=50,ymax=100.0,
	xlabel={Block Counts},xlabel style = {font = \Large},
 	point meta ={y*100},
 	ylabel={Avg. Model Quality}, ylabel style = {font=\Large},
	symbolic x coords={1,4,16,49,64,196,256},
	yticklabel=\pgfmathprintnumber\tick\%,yticklabel={\pgfmathparse{\tick}\pgfmathprintnumber{\pgfmathresult}\%},
	legend columns=2,legend style={at={(0.5,1.0)},draw=none,anchor=south,
 		font=\Large},
	y tick label style = {font = \Large},
	x tick label style = {font = \Large}
	]
	\addplot+[mark=*,error bars/.cd,
 	x dir=both
 	,y dir=both,y explicit]
 	table[x=block_count,y=accuracy, y error = std,col sep=comma]
 	{./data/block_variation_mnist.csv};
 	\addlegendentry{MNIST}	
	\addplot+[mark=square,color=red,draw opacity=0.6,error bars/.cd,
 	x dir=both,
 	y dir=both,y explicit]
 	table[x=block_count,y=accuracy,y error=std, col sep=comma]
 	{./data/block_variation_cifar.csv};
 	\addlegendentry{CIFAR}
	\addplot+[mark=square,color=black,draw opacity=0.6,error bars/.cd,
 	x dir=both,
 	y dir=both,y explicit]
 	table[x=block_count,y=accuracy,y error=std, col sep=comma]
 	{./data/block_variation_faces.csv};
	\addlegendentry{LFW}
 	\end{axis}
 \end{tikzpicture} 
\begin{tikzpicture}[scale=0.48]
 	\pgfplotsset{every axis legend/.append style={font=\small},every node near coord/.append style={font=\Large}}
 	\begin{axis}
 	[ymin=50,ymax=100.0,
	xlabel={Noise Levels},xlabel style = {font = \Large},
 	point meta ={y*100},
 	ylabel={Avg. Model Quality}, ylabel style = {font=\Large},
	xtick=data,
	yticklabel=\pgfmathprintnumber\tick\%,yticklabel={\pgfmathparse{\tick}\pgfmathprintnumber{\pgfmathresult}\%},
	legend columns=2,legend style={at={(0.5,1.0)},draw=none,anchor=south,
 		font=\Large},
	y tick label style = {font = \Large},
	x tick label style = {font = \Large}
	]	
	\addplot+[mark=*,error bars/.cd,
 	x dir=both
 	,y dir=both,y explicit]
 	table[x=noise_lvl,y=accuracy, y error = std,col sep=comma]
 	{./data/noise_variation_mnist.csv};
 	\addlegendentry{MNIST}	
	\addplot+[mark=square,color=red,draw opacity=0.6,error bars/.cd,
 	x dir=both,
 	y dir=both,y explicit]
 	table[x=noise_lvl,y=accuracy,y error=std, col sep=comma]
 	{./data/noise_variation_cifar.csv};
 	\addlegendentry{CIFAR-10}
	\addplot+[mark=square,color=black,draw opacity=0.6,error bars/.cd,
 	x dir=both,
 	y dir=both,y explicit]
 	table[x=noise_lvl,y=accuracy,y error=std, col sep=comma]
 	{./data/noise_variation_faces.csv};
 	\addlegendentry{LFW}
 	\end{axis}
 \end{tikzpicture}
	\vspace{-0.20cm}
 	\caption{Effect on Model Quality: Varying block counts (left). Varying noise levels (right)}
 	\label{fig:model_block_sizes}
	\vspace{-0.4cm}
\end{figure}
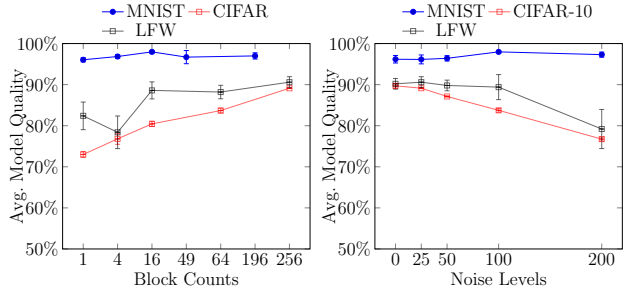
%\vspace{-0.2cm}

 \vspace{-0.35cm}
 \subsection{Attack Evaluation}
 \vspace{-0.15cm}
 %We assess the success of visual re-identification and class-membership attacks.
\subsubsection{Resilience to Visual Re-identification Attacks.} We observe that both the orthogonal and projection matrix based RMT successfully thwart the visual re-identification attacks (i.e. preserve high visual privacy) for all datasets as seen in Figure ~\ref{fig:perm} (left). Further permutation of the RMT blocks prominently increases the visual privacy for the sparse MNIST dataset whereas not so much for the denser CIFAR-10 and LFW datasets as seen in Figure ~\ref{fig:perm} (right). 

Figure ~\ref{fig:noise} (left) shows that the variation of the block sizes or the block counts does not affect visual privacy much for the denser datasets of CIFAR-10 and LFW. However, we observe a detectable drop in visual privacy for MNIST when using a single RMT for the entire image. On the other hand, the introduction of the additive noise does not seem to reduce the effectiveness of the DNN-examiners in compromising visual privacy as seen in Figure ~\ref{fig:noise} (right). As we observe high visual privacy across the parameter settings mostly,  we can be flexible in choosing the parameters that maximize the model quality. As Figure \ref{fig:model_block_sizes} shows, in general, we can choose smaller block sizes (larger block counts) and smaller noise levels to achieve better model quality. 

\begin{figure}[h]
\vspace{-0.4cm}
\centering
	\begin{tikzpicture}[scale=0.48]
	\begin{axis}[
 	ybar=-2pt, bar shift = 7pt,enlarge x limits=0.2,
 	bar width=12pt,
	log basis y={10},
 	ymin=50, ymax=100,
 	axis x line*=bottom, ylabel near ticks, yticklabel pos=left, yticklabel style={font=\Large},
	 y axis line style={opacity=50},
 	legend columns=2,legend style={at={(0.5,0.99)},draw=none,anchor=south,
 		font=\Large},
 	ylabel={Avg. Visual Privacy},
 	ylabel style={font=\Large},
 	symbolic x coords={MNIST Orth, MNIST Proj, blank, CIFAR Orth, CIFAR Proj, blank2,LFW Orth, LFW Proj},
	xticklabels={,MNIST, ,CIFAR-10, LFW},
	yticklabel=\pgfmathprintnumber\tick\%,yticklabel={\pgfmathparse{\tick}\pgfmathprintnumber{\pgfmathresult}\%}]
 	  \addplot[error bars/.cd, y dir=both,y explicit][draw=black, fill=blue,fill opacity=0.85,postaction={pattern=north east lines}] coordinates {
	  	(MNIST Orth,90.27)+-(3.33,3.33)};
	  \addplot[error bars/.cd, y dir=both,y explicit][draw=black, fill=yellow,fill opacity=0.85,postaction={pattern=grid}] coordinates {
		(MNIST Proj,91.03)+-(2.69,2.69)};
	 \addplot[error bars/.cd, y dir=both,y explicit][draw=black, fill=yellow,fill opacity=0.85,postaction={pattern=grid}] coordinates {
		(blank,0)};
	\addplot[error bars/.cd, y dir=both,y explicit][draw=black, fill=blue,fill opacity=0.85,postaction={pattern=north east lines}] coordinates {
	  	(CIFAR Orth,89.44)+-(0,0)};
	  \addplot[error bars/.cd, y dir=both,y explicit][draw=black, fill=yellow,fill opacity=0.85,postaction={pattern=grid}] coordinates {
		(CIFAR Proj,87.71)+-(0.57,0.57)};
	\addplot[error bars/.cd, y dir=both,y explicit][draw=black, fill=yellow,fill opacity=0.85,postaction={pattern=grid}] coordinates {
		(blank2,0)};
	\addplot[error bars/.cd, y dir=both,y explicit][draw=black, fill=blue,fill opacity=0.85,postaction={pattern=north east lines}] coordinates {
	  	(LFW Orth,98.0)+-(1.30,1.30)};
	  \addplot[error bars/.cd, y dir=both,y explicit][draw=black, fill=yellow,fill opacity=0.85,postaction={pattern=grid}] coordinates {
		(LFW Proj,86.0)+-(0.74,0.74)};
	\legend{Orthogonal, Projection} 
 	\end{axis}
 	\end{tikzpicture}
	\begin{tikzpicture}[scale=0.48]
	\begin{axis}[
 	ybar=-2pt, bar shift = 7pt,enlarge x limits=0.2,
 	bar width=12pt,
	log basis y={10},
 	ymin=50, ymax=100,
 	axis x line*=bottom, ylabel near ticks, yticklabel pos=left, yticklabel style={font=\Large},
	 y axis line style={opacity=50},
 	legend columns=2,legend style={at={(0.5,0.99)},draw=none,anchor=south,
 		font=\Large},
 	ylabel={Avg. Visual Privacy},
 	ylabel style={font=\Large},
 	symbolic x coords={MNIST Perm, MNIST NoPerm, blank, CIFAR Perm, CIFAR NoPerm, blank2,LFW Perm, LFW NoPerm},
	xticklabels={,MNIST, ,CIFAR-10, LFW},
	yticklabel=\pgfmathprintnumber\tick\%,yticklabel={\pgfmathparse{\tick}\pgfmathprintnumber{\pgfmathresult}\%}]
 	  \addplot[error bars/.cd, y dir=both,y explicit][draw=black, fill=blue,fill opacity=0.85,postaction={pattern=north east lines}] coordinates {
	  	(MNIST Perm,90.27)+-(3.33,3.33)};
	  \addplot[error bars/.cd, y dir=both,y explicit][draw=black, fill=yellow,fill opacity=0.85,postaction={pattern=grid}] coordinates {
		(MNIST NoPerm,77.43)+-(5.20,5.20)};
	 \addplot[error bars/.cd, y dir=both,y explicit][draw=black, fill=yellow,fill opacity=0.85,postaction={pattern=grid}] coordinates {
		(blank,0)};
	\addplot[error bars/.cd, y dir=both,y explicit][draw=black, fill=blue,fill opacity=0.85,postaction={pattern=north east lines}] coordinates {
	  	(CIFAR Perm,89.63)+-(0.0,0.0)};
	  \addplot[error bars/.cd, y dir=both,y explicit][draw=black, fill=yellow,fill opacity=0.85,postaction={pattern=grid}] coordinates {
		(CIFAR NoPerm,89.64)+-(3.18,3.18)};
			 \addplot[error bars/.cd, y dir=both,y explicit][draw=black, fill=yellow,fill opacity=0.85,postaction={pattern=grid}] coordinates {
		(blank2,-5)};
	\addplot[error bars/.cd, y dir=both,y explicit][draw=black, fill=blue,fill opacity=0.85,postaction={pattern=north east lines}] coordinates {
	  	(LFW Perm,98.00)+-(1.31,1.31)};
	  \addplot[error bars/.cd, y dir=both,y explicit][draw=black, fill=yellow,fill opacity=0.85,postaction={pattern=grid}] coordinates {
		(LFW NoPerm,98.00)+-(0.61,0.61)};
	\legend{Permutation, No permutation} 
 	\end{axis}
 	\end{tikzpicture}
	\vspace{-0.2cm}
 	\caption{Effect on Visual Privacy: Orthogonal vs. Projection (left). Permutation (right)}
	\vspace{-0.4cm}	
 	\label{fig:perm}
\end{figure}
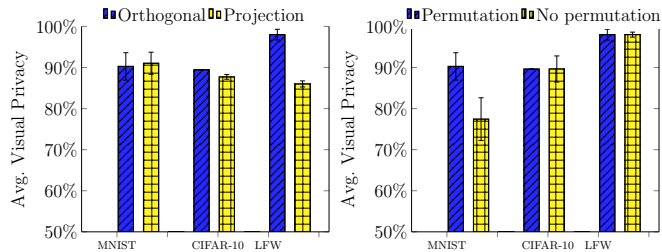
\vspace{-0.65cm}
\begin{figure}[h]
\vspace{-0.4cm}
\centering
\begin{tikzpicture}[scale=0.48]
 	\pgfplotsset{every axis legend/.append style={font=\small},every node near coord/.append style={font=\Large}}
 	\begin{axis}
 	[ymin=80.0,ymax=100.0,
	xlabel={Block Counts},xlabel style = {font = \Large},
 	point meta ={y*100},
 	ylabel={Avg. Visual Privacy}, ylabel style = {font=\Large},
	symbolic x coords={1,4,16,49,64,196,256},
	yticklabel=\pgfmathprintnumber\tick\%,yticklabel={\pgfmathparse{\tick}\pgfmathprintnumber{\pgfmathresult}\%},
	legend columns=2,legend style={at={(0.5,1.0)},draw=none,anchor=south,
 		font=\Large},
	y tick label style = {font = \Large},
	x tick label style = {font = \Large}
	]
	\addplot+[mark=*,blue,error bars/.cd,
 	x dir=both,y dir=both,y explicit]
 	table[x=block_count,y=accuracy,y error=std,col sep=comma]
 	{./data/block_variation_mnist_viz_priv.csv};
 	\addlegendentry{MNIST}
 	\addplot+[mark=x,red,draw opacity=0.6,error bars/.cd,
 	x dir=both,
 	y dir=both,y explicit]
 	table[x=block_count,y=accuracy,y error=std, col sep=comma]
 	{./data/block_variation_cifar_viz_priv.csv};
 	\addlegendentry{CIFAR-10}
	\addplot+[mark=square,color=black,draw opacity=0.6,error bars/.cd,
 	x dir=both,
 	y dir=both,y explicit]
 	table[x=block_count,y=accuracy,y error=std, col sep=comma]
 	{./data/block_variation_faces_viz_priv.csv};
	\addlegendentry{LFW}
 	\end{axis}
 	\end{tikzpicture}
\begin{tikzpicture}[scale=0.48]
 	\pgfplotsset{every axis legend/.append style={font=\small},every node near coord/.append style={font=\Large}}
 	\begin{axis}
 	[ymin=80.0,ymax=100,
	xlabel={Noise Levels},xlabel style = {font = \Large},
 	point meta ={y*100},
 	ylabel={Avg. Visual Privacy}, ylabel style = {font=\Large},
	%symbolic x coords={2,4,8,16,32},
	xtick=data,
	yticklabel=\pgfmathprintnumber\tick\%,yticklabel={\pgfmathparse{\tick}\pgfmathprintnumber{\pgfmathresult}\%},
	legend columns=2,legend style={at={(0.5,1.0)},draw=none,anchor=south,
 		font=\Large},
	y tick label style = {font = \Large},
	x tick label style = {font = \Large}
	]
	\addplot+[mark=*,blue,error bars/.cd,
 	x dir=both
 	,y dir=both,y explicit]
 	table[x=noise_lvl,y=accuracy,y error=std,col sep=comma]
 	{./data/noise_variation_mnist_viz_priv.csv};
 	\addlegendentry{MNIST}
 	\addplot+[mark=x,red,draw opacity=0.6,error bars/.cd,
 	x dir=both,
 	y dir=both,y explicit]
 	table[x=noise_lvl,y=accuracy,y error=std,col sep=comma]
 	{./data/noise_variation_cifar_viz_priv.csv};
 	\addlegendentry{CIFAR-10.}	
	\addplot+[mark=square,color=black,draw opacity=0.6,error bars/.cd,
 	x dir=both,
 	y dir=both,y explicit]
 	table[x=noise_lvl,y=accuracy,y error=std, col sep=comma]
 	{./data/noise_variation_faces_viz_priv.csv};
 	\addlegendentry{LFW}
 	\end{axis}
 	\end{tikzpicture}
	%\vspace{-0.40cm}
 	\caption{Effect on Visual Privacy:Varying block counts (left). Varying noise levels (right)}
 	\label{fig:noise}
	\vspace{-0.4cm}
\end{figure}
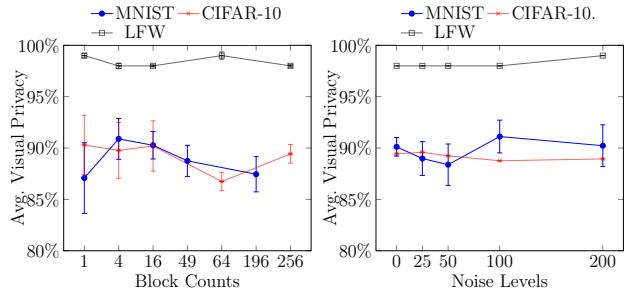
%\vspace{-0.5cm}

\vspace{-0.4cm}
\subsubsection{Resilience to Class-membership Attacks}
Next, we examine the resilience of our method to the class-membership attack. Following the empirical method described in Section ~\ref{subsec:empirical}, we measure the class-wise Fano factors for the prediction output probabilities for both the in-training and out-training datasets. We use different datasets to test how class-membership attacks perform on the DNN models. Specifically, we partition the datasets by class and then feed the images in the same class into the model, one class at a time. We then summarize the output distribution with the Fano factor. Without applying image disguising, we observe in Figure  ~\ref{fig:fano_org}, the Fano factor values for the in-training classes are clearly distinguishable from those of the out-training classes, with statistically significant margins (p-value $\leq 0.001$). In contrast, for the transformed models, we observe in Figure ~\ref{fig:fano_pert} the in-training classes and out-training classes are not distinguishable from each another - a small difference between average values with p-value $>0.5$. 

% why 
\vspace{-0.25cm}
\begin{figure} [h]
\vspace{-0.4cm}
\centering
\includegraphics[width= 0.33\linewidth]{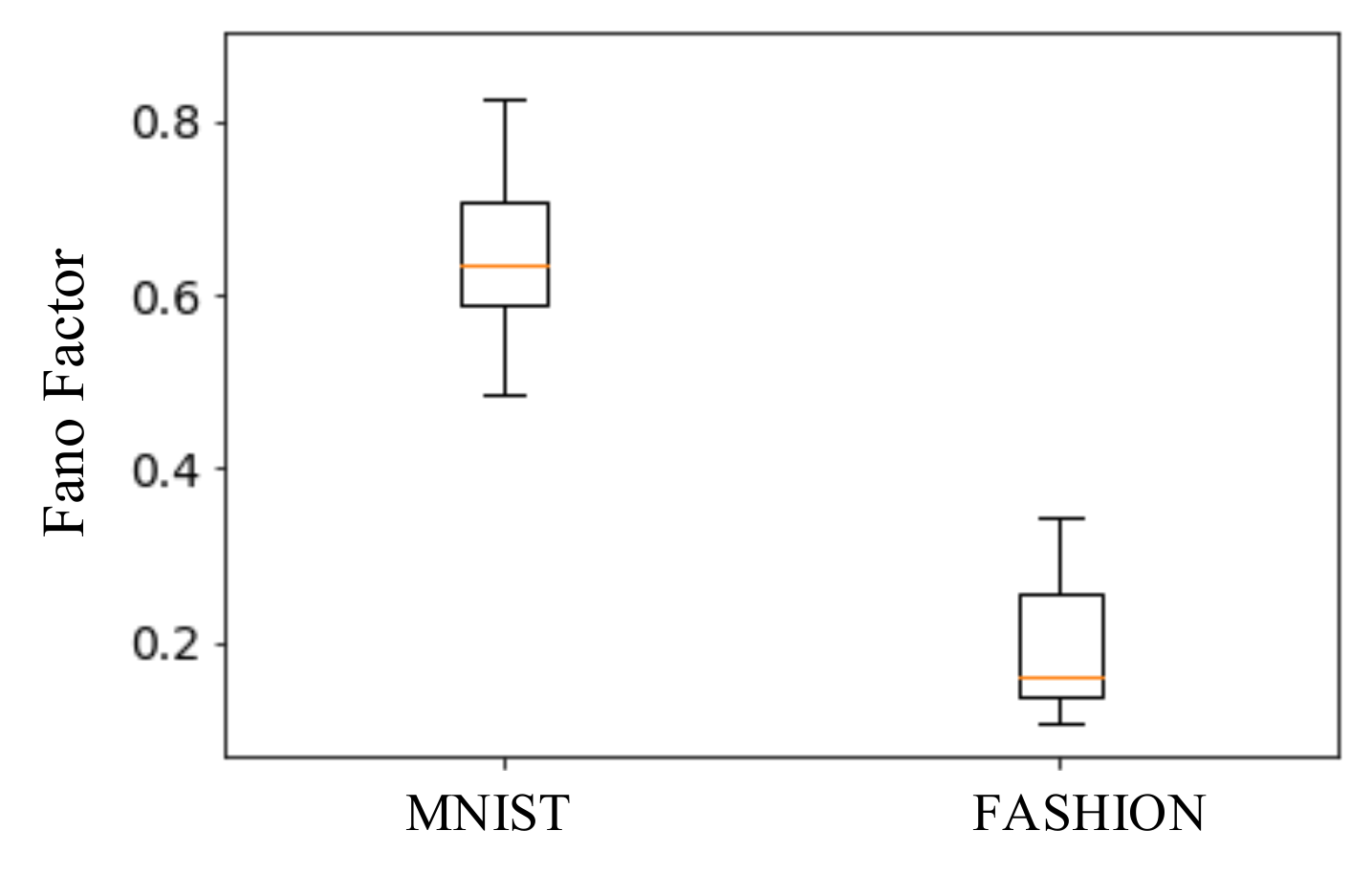}
\includegraphics[width= 0.33\linewidth]{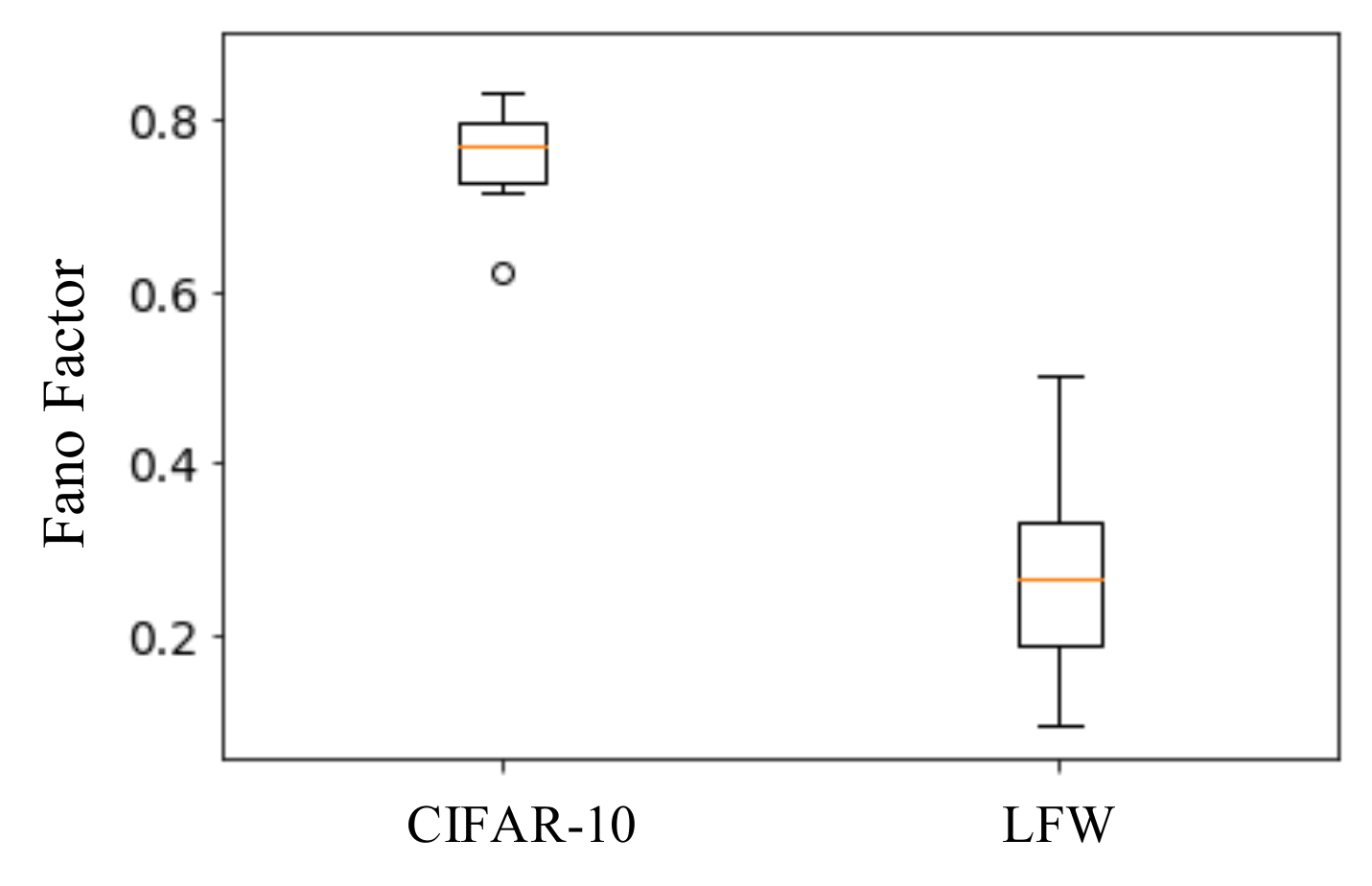}
\vspace{-0.35cm}
\caption{Effective class-membership attack on the unprotected models. In-training class-wise Fano factor is significantly higher. }
\label{fig:fano_org}
\vspace{-0.45cm}
\end{figure}
\vspace{-0.53cm}
\begin{figure} [h]
\vspace{-0.4cm}
\centering
\includegraphics[width=0.33\linewidth]{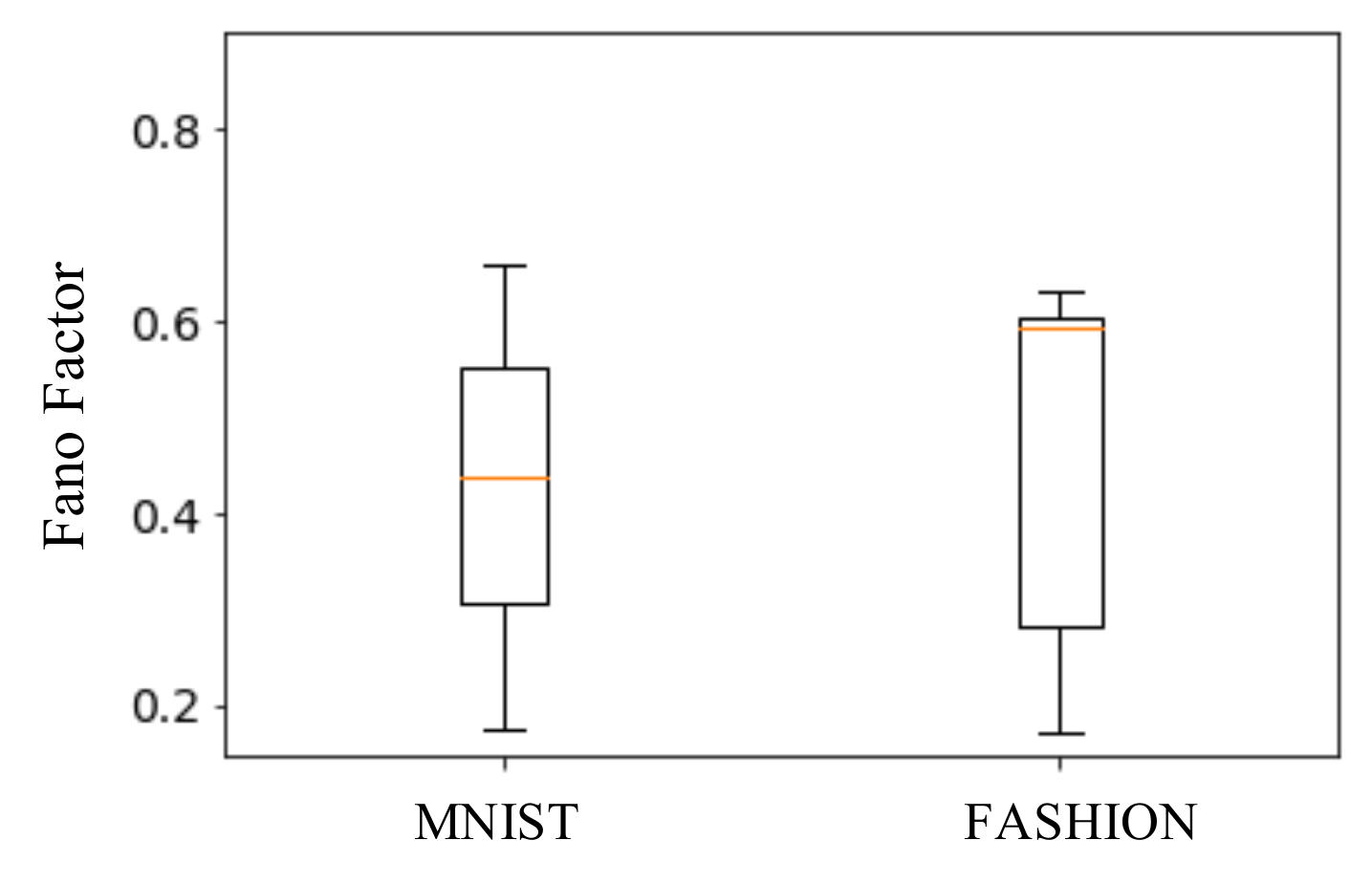}
\includegraphics[width=0.33\linewidth]{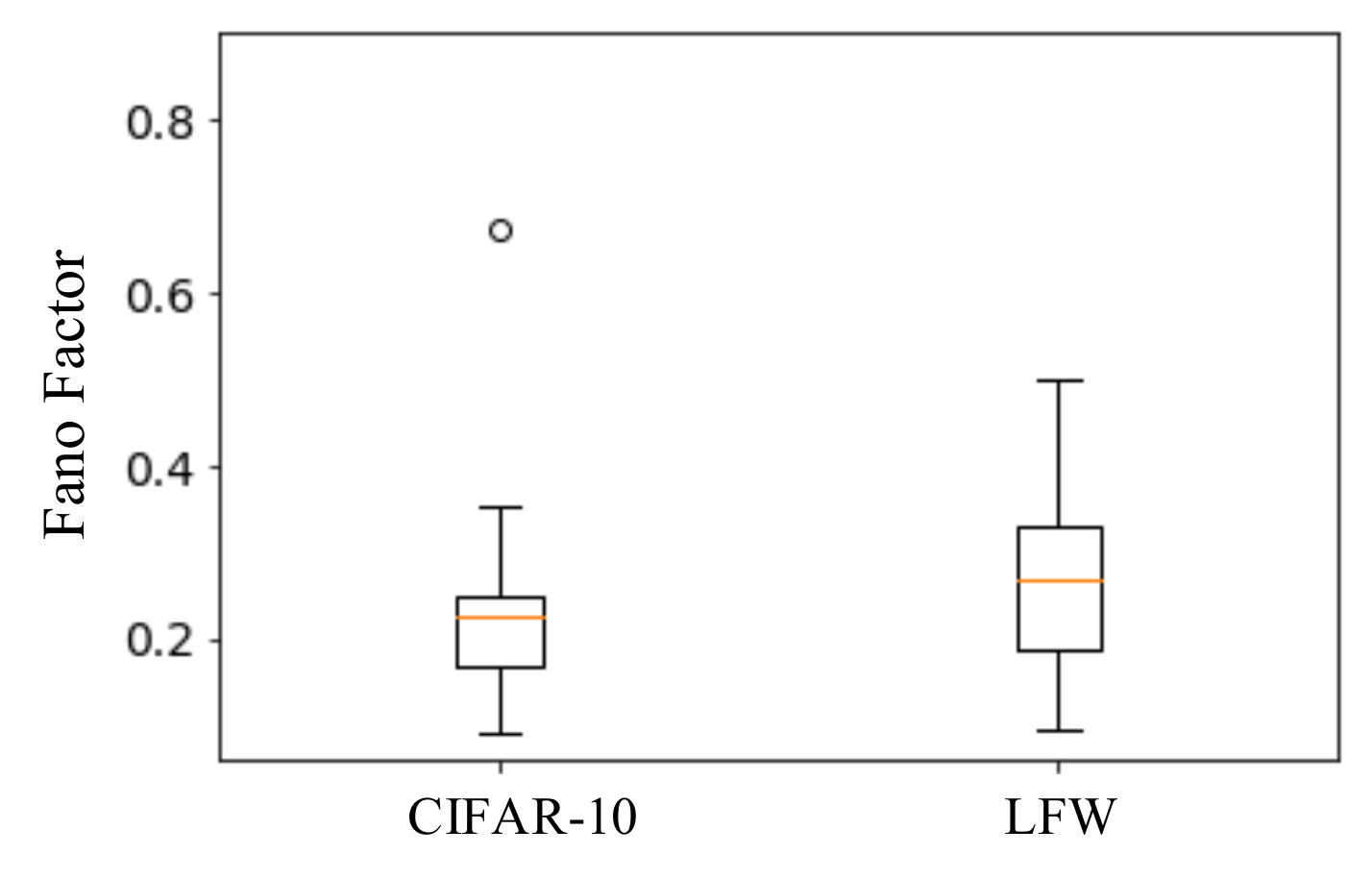}
\vspace{-0.35cm}
\caption{Ineffective class-membership attack on the protected models. In-training class-wise Fano factor is indistinguishable from that of  Out-training. }
\label{fig:fano_pert}
\vspace{-0.3cm}
\end{figure}

\vspace{-0.35cm}

\textbf{Discussion.}
At its current state, Disguised-Nets only considers learning DNN models that classify images to individual labels. It will be important to assess if these results carry over to other learning objectives such as multi-label classification and regression. Furthermore, adapting our disguising mechanism in transfer learning, which has proven extremely useful in building powerful models, is a challenging yet interesting task. 
Lastly, we would like to explore in an expansion of this work, how closely a malicious adversary can estimate the $R_i$ matrices in RMT given some leaked pairs of original and disguised images. 

\vspace{-0.35cm}
\section{Related Work}\label{sec:related_work}
\vspace{-0.4cm}

Fan et al. ~\cite{fan18} applied a differentially private mechanism to hide certain pixels in images for image pixelation, however, the obfuscated images are visually identifiable from the global perspective. Li et al. ~\cite{li17} propose learning shallower neural networks locally by the data owner and sharing the intermediate representation to the cloud for further learning. Mao et al. ~\cite{mao18} propose a similar strategy for face detection problem. They let the data owner evaluate the first layer of the DNN and apply a differentially private noise to the output. However, both approaches reveal the visually identifiable features of the images. 

The most related crypto approach for training DNNs is Mohassel et al. ~\cite{mohassel17} for SGD-based logistic regression and neural networks. It is based on randomized secret sharing, additively homomorphic encryption, and garbled circuits. However, the framework is very expensive even for small-scale neural networks. 

Abadi et al. ~\cite{abadi16} and Shokri et al. \cite{reza15} propose training differentially private DNN models that hides inclusion or exclusion of individual images in the training data from the model consumers with noisy SGD update algorithms. These techniques are unsuitable in the outsourced setting for DNN learning as they do not directly protect the content of the images or the learned models. Similarly, they present a significant trade-off between model quality. 

A set of research focuses on the privacy-preserving evaluation of DNN models, which is easier to build and less costly than privacy-preserving DNN learning frameworks. Nathan et al. ~\cite{xie14} present the homomorphic encryption based CryptoNets framework for evaluating a DNN with encrypted input data. Similarly, Rouhani et al. ~\cite{rouhani18}  propose a garbled circuit based DNN evaluation protocol. 

Our idea of class-membership attack is slightly related to the membership inference addressed by ~\cite{shokri16}, however completely a different concept. Membership inference attack aims to determine inclusion or exclusion of exact data points in the training set whereas class-membership attack determines inclusion or exclusion of a certain kind or category of images in the training dataset. Shokri et al. ~\cite{shokri19} assess the attack specifically on DNN models. With Disguised-Nets, this attack becomes irrelevant as it is impossible for an adversary to design and launch the attack without knowing the exact transformation keys we deploy.

Fredrikson et al. \cite{fredrikson14} show that it is possible to reverse engineer a machine learning model to explore the private training data the model was trained with a model inversion attack (MIA). The success of the MIA attack depends on unrestrained access to the target machine learning models. With a high level of visual privacy and the link between the transformed and original images broken by the RMT parameters, Disguised-Nets need not worry about this attack as the generated images are also in the transformed space.

\vspace{-0.45cm}
\section{Conclusion}\label{sec:conclusion}
\vspace{-0.4cm}
While using cloud resources for deep learning has been an economical option, only a few studies address the related privacy concerns. In this paper, we identify two types of attacks on outsourced deep learning: the visual re-identification attack and the class-membership attack, which none of the existing candidate solutions can satisfactorily address. We propose our image disguising mechanisms: Disguised-Nets for privacy-preserving deep learning in the outsourced setting. It employs a combination of block-wise secret permutation and multidimensional transformations on each image while preserving a certain utility that the deep learning algorithms can pick up. Experimental results show that the Disguised-Nets approach preserves the model quality surprisingly well. It is also resilient to the visual re-identification and the class-membership attacks.  

\vspace{-0.45cm}
\bibliographystyle{abbrv}
\vspace{-0.3cm}
\bibliography{./paper}

\begin{thebibliography}{10}

\bibitem{liu17dnn}
A survey of deep neural network architectures and their applications.
\newblock {\em Neurocomputing}, 234:11 -- 26, 2017.

\bibitem{abadi16}
M.~Abadi, A.~Chu, I.~Goodfellow, H.~B. McMahan, I.~Mironov, K.~Talwar, and
  L.~Zhang.
\newblock Deep learning with differential privacy.
\newblock 2016.

\bibitem{curtis84}
M.~L. Curtis.
\newblock {\em Matrix Groups}.
\newblock Universitext, 1984.

\bibitem{fan18}
L.~Fan.
\newblock Image pixelization with differential privacy.
\newblock In {\em Data and Applications Security and Privacy {XXXII} - 32nd
  Annual {IFIP} {WG} 11.3 Conference, DBSec 2018, Bergamo, Italy, July 16-18,
  2018, Proceedings}, pages 148--162, 2018.

\bibitem{fredrikson14}
M.~Fredrikson, E.~Lantz, S.~Jha, S.~Lin, D.~Page, and T.~Ristenpart.
\newblock Privacy in pharmacogenetics: An end-to-end case study of personalized
  warfarin dosing.
\newblock In {\em 23rd USENIX Security Symposium USENIX Security 14}, pages
  17--32, San Diego, CA, 2014. USENIX Association.

\bibitem{gallier00}
J.~Gallier.
\newblock {\em Geometric Methods and Applications for Computer Science and
  Engineering}.
\newblock Springer-Verlag, New York, 2000.

\bibitem{kaiming15}
K.~{He}, X.~{Zhang}, S.~{Ren}, and J.~{Sun}.
\newblock Deep residual learning for image recognition.
\newblock In {\em 2016 IEEE Conference on Computer Vision and Pattern
  Recognition (CVPR)}, pages 770--778, June 2016.

\bibitem{krizhevsky17}
A.~Krizhevsky, I.~Sutskever, and G.~E. Hinton.
\newblock Imagenet classification with deep convolutional neural networks.
\newblock {\em Commun. ACM}, 60(6):84--90, May 2017.

\bibitem{li17}
M.~Li, L.~Lai, N.~Suda, V.~Chandra, and D.~Z. Pan.
\newblock Privynet: {A} flexible framework for privacy-preserving deep neural
  network training with {A} fine-grained privacy control.
\newblock {\em CoRR}, abs/1709.06161, 2017.

\bibitem{liu06tkde}
K.~Liu, H.~Kargupta, and J.~Ryan.
\newblock Random projection-based multiplicative data perturbation for privacy
  preserving distributed data mining.
\newblock {\em IEEE Transactions on Knowledge and Data Engineering (TKDE)},
  18(1):92--106, 2006.

\bibitem{mao18}
Y.~Mao, S.~Yi, Q.~Li, J.~Feng, F.~Xu, and S.~Zhong.
\newblock A privacy-preserving deep learning approach for face recognition with
  edge computing.
\newblock In {\em {USENIX} Workshop on Hot Topics in Edge Computing (HotEdge
  18)}, Boston, MA, 2018. {USENIX} Association.

\bibitem{mohassel17}
P.~Mohassel and Y.~Zhang.
\newblock Secureml: A system for scalable privacy-preserving machine learning.
\newblock In {\em 2017 IEEE Symposium on Security and Privacy (SP)}, 2017.

\bibitem{shokri19}
M.~Nasr, R.~Shokri, and A.~Houmansadr.
\newblock {Comprehensive Privacy Analysis of Deep Learning: Stand-alone and
  Federated Learning under Passive and Active White-box Inference Attacks}.
\newblock In {\em The $40^{th}$ IEEE Symposium on Security and Privacy
  (Oakland)}, 2019.

\bibitem{rouhani18}
B.~D. Rouhani, M.~S. Riazi, and F.~Koushanfar.
\newblock Deepsecure: Scalable provably-secure deep learning.
\newblock In {\em Proceedings of the 55th Annual Design Automation Conference},
  DAC '18, pages 2:1--2:6, New York, NY, USA, 2018. ACM.

\bibitem{reza15}
R.~Shokri and V.~Shmatikov.
\newblock Privacy-preserving deep learning.
\newblock In {\em Proceedings of the 22nd ACM SIGSAC Conference on Computer and
  Communications Security}, 2015.

\bibitem{shokri16}
R.~{Shokri}, M.~{Stronati}, C.~{Song}, and V.~{Shmatikov}.
\newblock Membership inference attacks against machine learning models.
\newblock In {\em 2017 IEEE Symposium on Security and Privacy (SP)}, pages
  3--18, May 2017.

\bibitem{vempala05}
S.~S. Vempala.
\newblock {\em The Random Projection Method}.
\newblock American Mathematical Society, 2005.

\bibitem{xie14}
P.~Xie, M.~Bilenko, T.~Finley, R.~Gilad{-}Bachrach, K.~E. Lauter, and
  M.~Naehrig.
\newblock Crypto-nets: Neural networks over encrypted data.
\newblock {\em CoRR}, abs/1412.6181, 2014.

\end{thebibliography}

\end{document}